\documentclass[10pt,twocolumn,letterpaper]{article}
\usepackage[pagenumbers]{cvpr}
\usepackage{amsthm}
\usepackage{booktabs}
\usepackage{multirow}
\usepackage{float}
\newtheorem{definition}
{Definition}

\usepackage{graphicx}

\usepackage{graphicx}
\usepackage{wrapfig}

\definecolor{linkblue}{rgb}{0.21,0.49,0.74}
\usepackage[breaklinks,colorlinks,allcolors=linkblue]{hyperref}
\hypersetup{pdftitle={Who Drives the Probability Game of VLMs? A Temporal Causal Drive Evaluation Framework},pdfauthor={Shuyao Xiao, Shengling Wang, Haoyu Niu, Xuan Chen, Ke Chao, Changwei Xu, Xinran Duan, Chaoyong Jiang},pdfsubject={Preprint}}
\title{Who Drives the Probability Game of VLMs?\\ A Temporal Causal Drive Evaluation Framework}
\author{%
Shuyao Xiao\textsuperscript{1,2}\thanks{This work was done during an internship at Kuaishou Technology.}\quad
Shengling Wang\textsuperscript{1}\thanks{Corresponding author.}\quad
Haoyu Niu\textsuperscript{2}\quad
Xuan Chen\textsuperscript{1}\\
Ke Chao\textsuperscript{1}\quad
Changwei Xu\textsuperscript{2}\quad
Xinran Duan\textsuperscript{1}\quad
Chaoyong Jiang\textsuperscript{1}\\[6pt]
\textsuperscript{1}College of Artificial Intelligence, Beijing Normal University\\
\textsuperscript{2}Kuaishou Technology\\[6pt]
{\ttfamily\small xiaoshuyao@mail.bnu.edu.cn, wangshengling@bnu.edu.cn, niuhaoyu@kuaishou.com}}

\begin{document}
\maketitle
\begin{abstract}
Evaluating vision-language models (VLMs) requires knowing both whether an answer is correct and what information supports it. Current answer-level scores offer only limited insight into the respective contributions of visual inputs, questions, and generated prefixes, or how these contributions evolve during generation. Conventional evaluation scores  can therefore hide differences in source support, making it harder to assess visual, prompt or prefix  understanding and diagnose generation errors. To address this gap, we propose a temporal causal drive evaluation framework. We use a structural causal model to describe dependencies among information sources. Within this model, we apply interventions and backdoor adjustment to derive three metrics: Prefix Causal Drive (PCD), Visual Causal Drive (VCD), and Question Causal Drive (QCD). These metrics use token-level output probabilities under different source conditions, without requiring access to attention weights or hidden states. Experiments show that our metrics respond to controlled changes in questions and visual evidence, revealing differences in source support beyond answer-matching scores. Cross-model comparisons and trajectory analyses show their value for understanding how models generate answers.
\end{abstract}    
\section{Introduction}
\label{sec:intro}

What drives a vision-language model's answer? Vision-language models (VLMs) have advanced in visual question answering and multimodal reasoning~\cite{yin2024survey}. Evaluating these models involves assessing both answer correctness and the evidence supporting their answers~\cite{jacovi-goldberg-2020-towards}. Visual inputs and questions condition answer generation, while previously generated tokens provide context for subsequent outputs~\cite{alayrac2022flamingo,bengio2003neural}. A high answer score does not establish sufficient visual support~\cite{NEURIPS2024_2f8ee6a3}. Answer correctness and the attribution of generated content to its input sources are distinct aspects of model behavior~\cite{cohen-wang2024contextcite}. Research on reasoning faithfulness further shows that generated explanations need not reflect the factors that actually drive predictions~\cite{jin-etal-2026-mirage}. These findings motivate examining how support from different information sources evolves throughout generation.

Existing evaluation methods assess multimodal capabilities across different tasks, domains, and output modalities~\cite{Zhang2025}. MMBench~\cite{liu2024mmbench} enhances outcome assessment robustness through multiple-choice evaluation and CircularEval, while MM-Vet~\cite{mmvet} introduces an LLM-based evaluator for diverse open-ended responses. MAVIS~\cite{song2026mavis} evaluates long-form answers in terms of informativeness, groundedness through citations, and fluency. VLMs Are Biased~\cite{vlms_are_biased} uses answer accuracy and bias rate to assess whether model predictions on counterfactual images match visual facts or predefined biased answers. However, these scores share a common limitation: they primarily evaluate generated answers and their relationships to external references, making it difficult to distinguish different source-support patterns underlying similar outcomes, let alone characterize how these patterns evolve during generation. This leaves a diagnostic gap between answer scores and the process of answer formation. To address this gap, we track changes in the support provided by visual inputs, questions, and generated prefixes along the same generation trajectory, complementing outcome evaluation with source-specific process information.

Achieving this goal poses two core challenges. First, information sources depend on one another, making it difficult to separate observational correlations from causal effects~\cite{feder2022causal}. In visual question answering, questions are typically constructed around given images~\cite{antol2015vqa}. Generated prefixes are themselves conditioned on both the visual input and the question~\cite{NEURIPS2023_6dcf277e}. Consequently, interpreting conditional associations as the effect of a single source may conflate shared upstream influences and distinct causal paths~\cite{pearl2010causal}. Second, autoregressive generation is a dynamic process. As an answer unfolds, the generated prefix grows and changes the context for subsequent tokens~\cite{xiong2025autoregressive}. The effects of visual inputs, questions, and prefixes may therefore vary across generation positions. A single answer-level score cannot readily capture these changes. Source analysis must account for dependencies among information sources and track how their effects evolve throughout generation.

To address these challenges, we propose a temporal causal drive evaluation framework for VLM generation. Using a structural causal model (SCM), we formalize answer generation as a dynamic process jointly influenced by visual inputs, questions, and generated prefixes. Under the specified causal graph and identification assumptions, we derive source-specific interventional quantities through interventions and backdoor adjustment~\cite{pearl2010causal}. Building on these quantities, we construct Prefix Causal Drive (PCD), Visual Causal Drive (VCD), and Question Causal Drive (QCD) to characterize generation support under the corresponding interventional conditions. These metrics rely on token-level output probabilities under different source conditions, without requiring access to attention weights or hidden states. By tracking metric changes along fixed answer trajectories, the framework extends source evaluation from aggregate outcomes to the generation process, providing a unified temporal perspective on the roles of visual evidence, question information, and generation history.

We evaluate the framework on several datasets through controlled input modifications, cross-model comparisons, and generation trajectory analyses. The results show that these metrics respond to changes in question content and visual evidence, revealing differences in source support beyond answer-matching scores. This allows source analysis to complement existing evaluations of answer quality and biased behavior, further characterizing generation patterns under different information conditions. Our contributions are summarized as follows:

\noindent\textbf{(1)} We propose a causal and temporal evaluation framework for VLM generation, modeling autoregressive decoding as a dynamic process jointly influenced by visual inputs, questions, and generated prefixes. We construct three metrics, PCD, VCD, and QCD, to track generation support under different source conditions through interventional analysis.

\noindent\textbf{(2)} We investigate how support from different information sources evolves during generation and examine the responses of the proposed metrics through controlled input modifications and cross-model analyses. The framework captures process-level patterns that conventional outcome-based metrics struggle to reveal, providing a finer-grained diagnostic perspective on VLM generation behavior.
\section{Related Work}
\label{sec:related works}

Existing VLM evaluation methods assess task performance, answer quality, and evidence support. For image captioning and free-form generation, commonly used lexical and semantic similarity metrics include BLEU~\cite{papineni2002bleu}, ROUGE~\cite{lin2004rouge}, METEOR~\cite{banerjee2005meteor}, CIDEr~\cite{vedantam2015cider}, SPICE~\cite{anderson2016spice}, and BERTScore~\cite{zhang2019bertscore}. Visual question answering and multiple-choice benchmarks typically report answer accuracy~\cite{antol2015vqa,fu2023mme,liu2024mmbench,yue2024mmmu}, while MM-Vet~\cite{mmvet} employs an LLM-based evaluator for open-ended responses. MAVIS~\cite{song2026mavis} evaluates informativeness, citation groundedness, and fluency in long-form answers. VLMs Are Biased~\cite{vlms_are_biased} measures accuracy and bias rate on counterfactual images to examine whether predictions follow visual facts or predefined biased answers. These methods evaluate generated outputs against reference answers, supporting evidence, or expected behaviors. Our framework further quantifies support from visual inputs, questions, and generated prefixes, revealing source-specific differences behind similar answer scores.

Process-oriented evaluation examines reasoning steps and confidence trajectories. \citet{NEURIPS2023_ed3fea90} show that suggested answers can influence model predictions without being acknowledged in generated explanations. ReCEval~\cite{prasad2023receval} evaluates reasoning correctness and informativeness; Temporalizing Confidence~\cite{mao2025temporalizing} evaluates stepwise confidence using signal temporal logic; and THINK-Bench~\cite{li2025think} investigates reasoning efficiency and chain-of-thought quality. In multimodal settings, visually grounded VideoQA~\cite{xiao2024can} evaluates whether relevant video evidence supports an answer. Grounded Chain-of-Thought~\cite{wu2025grounded} links reasoning to visual regions and evaluates answer accuracy, grounding accuracy, and their consistency. These studies focus on reasoning quality, explanation faithfulness, and evidence alignment. Our framework tracks question, visual, and prefix causal drives at token positions along the same answer trajectory, showing how their effects evolve during generation.

Other work uses partial information decomposition to characterize unique, redundant, and synergistic multimodal information~\cite{liang2023quantifying,amit2025quantifying}, or expected relative entropy to quantify modality-specific information gain in multimodal generalized linear models~\cite{jin2026inference}. Studies also examine visual information processing within models~\cite{neo2024towards,liu2026seeing} and use module-level interventions to investigate hallucinations and modality priors~\cite{yang2025modular,zhou2025mitigating}. Visual information grounding uses next-token log-probability contrasts to strengthen visual support~\cite{favero2024multi}. Visual Contrastive Decoding contrasts output distributions under original and distorted images to reduce object hallucinations~\cite{Leng_2024_CVPR}. ContextCite attributes generated statements to context sources by fitting a surrogate model to response probabilities under context ablations~\cite{cohen-wang2024contextcite}. Our framework uses an SCM to represent dependencies among visual inputs, questions, and generated prefixes and define source-specific interventional quantities. Scoring a fixed answer trajectory enables direct comparison across source conditions and requires only output probabilities.
\section{Temporal Causal Framework}

We model answer generation in multimodal question answering as a temporal autoregressive process. 
The image (or video) and question text are first tokenized into sequences \cite{xiong2025autoregressive}. 
Let $V$ denote the visual token sequence and $Q$ denote the question token sequence. 
Given the visual evidence and the question, the model generates the answer autoregressively \cite{alayrac2022flamingo}. 
We use $A_{\le t}$ to denote the complete answer sequence up to step $t$, with $A_t$ denoting the current token at step $t$.
Thus, the generation process can be viewed as a sequential process where the answer  recursively grows until an end-of-sequence token is produced or a maximum length is reached.

\begin{figure}
    \centering
    \includegraphics[width=0.95\linewidth]{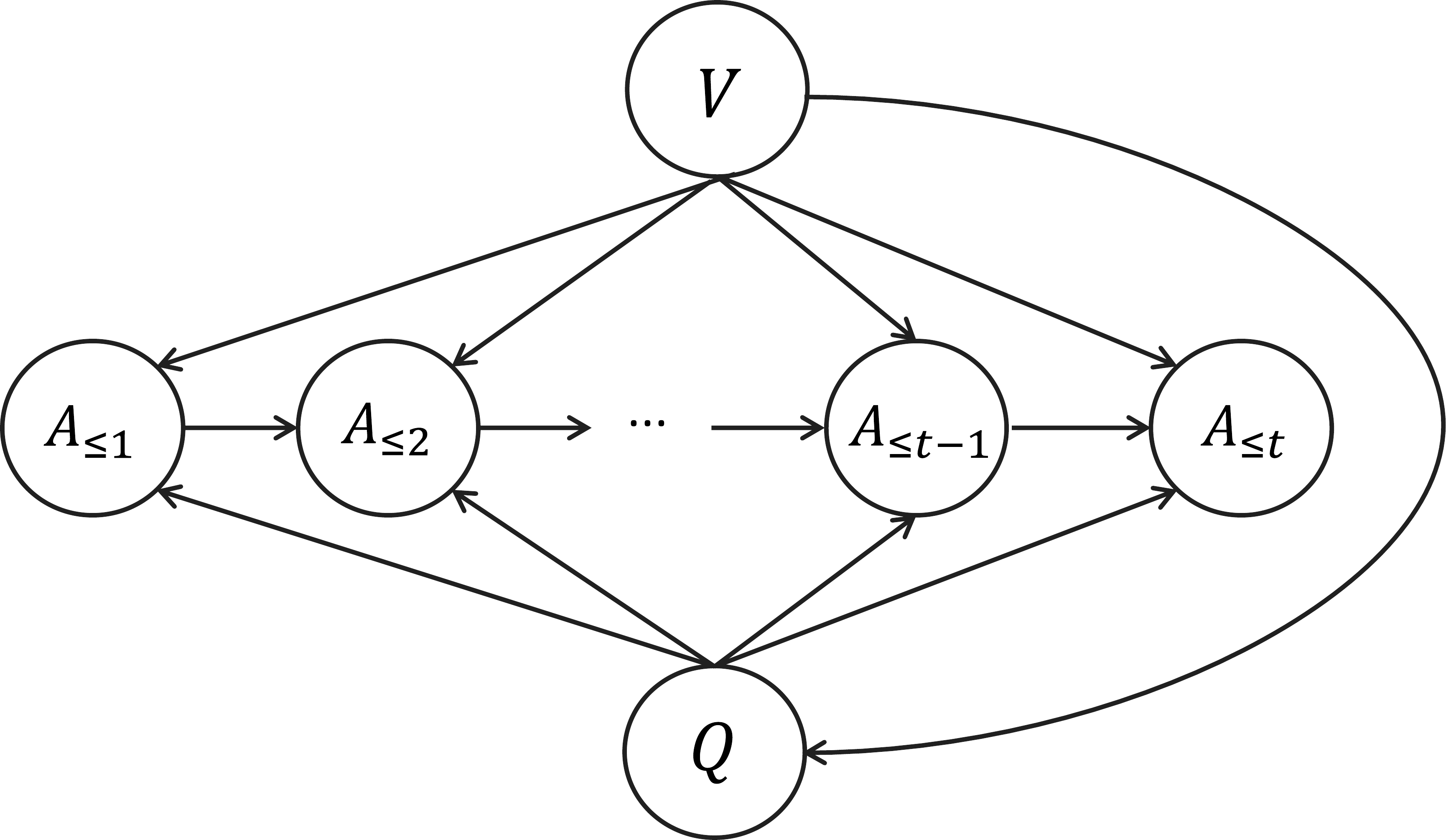}
    \caption{Structural causal graph of VLM generation. Here, $V$ denotes the visual input, $Q$ denotes the question or prompt, and $A_{\le t}$ denotes the generated answer at token step $t$.}
    \label{fig:scm}
    
\end{figure}

We formalize this process as the structural causal graph shown in Figure~\ref{fig:scm}. 
In this graph, the relation $A_{<t} \rightarrow A_t$ represents the autoregressive propagation of the answer over time; each current token is constrained by the existing prefix.
The edges
$Q \rightarrow \{A_{\le 1}, A_{\le 2}, \ldots, A_{\le t}\}$
indicate that the question continuously guides the semantic objective throughout generation, while
$V \rightarrow \{A_{\le 1}, A_{\le 2}, \ldots, A_{\le t}\}$
indicates that visual evidence provides external grounding for answer formation. 
In VQA datasets, questions are typically constructed around a given image~\cite{antol2015vqa}. 
Therefore, questions are typically not independent of the visual input but are constructed around the given image or video. 
We represent this dependence by the edge $V\rightarrow Q$.

In this graph, the current token $A_t$ is affected by multiple causal paths.
The first type consists of direct paths: $A_{<t} \rightarrow A_t$, $V \rightarrow A_t$, and $Q \rightarrow A_t$.
The second type includes confounding paths. For the prefix effect, $A_{<t}$ and $A_t$ share the upstream variables $Q$ and $V$, inducing the backdoor paths
$A_{<t}\leftarrow Q\rightarrow A_t$
and
$A_{<t}\leftarrow V\rightarrow A_t$ \cite{pearl2009causality}.
For the question effect, the path \(Q \leftarrow V \rightarrow A_t\) creates confounding because \(V\) affects both $Q$ and $A_t$, so the observed effect of $Q$ on $A_t$ is also biased.
As a result, the observed relations $A_{<t}\rightarrow A_t$ and $Q\rightarrow A_t$ do not by themselves identify the corresponding interventional effects on the current token.
Directly interpreting changes in conditional probabilities as the causal contribution of a given source conflates statistical correlation with causal influence.

A straightforward alternative is to estimate source effects by repeatedly ablating inputs and regenerating outputs. However, such counterfactual decoding is costly, unstable, and difficult to scale in large-generation settings. 
We therefore adopt the SCM to analyze the generation process. 
To distinguish causal influence from statistical association, we introduce the do-operator within the SCM framework \cite{pearl2009causality}. 
The do-operator represents an intervention that fixes a variable to a specified value and cuts off its incoming causal dependencies, allowing us to block backdoor paths analytically. 
We score the same generated sequence under different source conditions using teacher forcing~\cite{6795228}.

\label{subsec:causal effects}
Based on the SCM, we derive causal effects to quantify source-level drives in generation. 
Each effect captures the change in output probability from interventions, rather than just statistical association \cite{mlodozeniec2025position}. In the following expression, uppercase letters such as $Q$ and $V$ denote random variables, while lowercase letters such as $q$ and $v$ denote their observed values.
The adjustment distributions are estimated from empirical evaluation data, so the resulting effects should be interpreted for the benchmark population under evaluation.

We first consider the generated prefix \(A_{<t}\), whose influence on the current token \(A_t\) may be confounded by both the question \(Q\) and visual input \(V\).
According to the backdoor criterion \cite{pearl2009causality}, after intervening on the prefix \(do(A_{<t}=a_{<t})\), we obtain:
\begin{equation}
\begin{aligned}
&P(A_t=a_t \mid do(A_{<t}=a_{<t})) \\
&=
\sum_{v,q} P(V=v)P(Q=q\mid V=v) \\
&\quad\cdot
P(A_t=a_t\mid A_{<t}=a_{<t},V=v,Q=q).
\end{aligned}
\label{eq:prefix_effect}
\end{equation}
For the question effect, the intervention \(do(Q=q)\) cuts off the incoming edge \(V \rightarrow Q\), blocking the backdoor path via $V$. 
Since \(V\) confounds the relationship between \(Q\) and \(A_{\le t}\), we perform backdoor adjustment \cite{pearl2010causal} on the empirical distribution of \(V\) to block the backdoor path \(Q \leftarrow V \rightarrow A_{\le t}\). The causal effect of question $Q$ on $A_{\le t}$ is
\begin{equation}
\begin{aligned}
&P(A_{\le t}=a_{\le t} \mid do(Q=q)) \\
&= 
\sum_v
P(A_{\le t}=a_{\le t} \mid V=v, Q=q)\, P(V=v).
\label{eq:question_effect}
\end{aligned}
\end{equation}
To characterize the contribution of visual evidence independently of question information, we evaluate the visual intervention while fixing the question to a null prompt:
\begin{equation}
\begin{aligned}
&P(A_{\le t}=a_{\le t}
\mid do(V=v),do(Q=\varnothing))\\
&=
P(A_{\le t}=a_{\le t}
\mid V=v,Q=\varnothing).
\end{aligned}
\label{eq:visual_effect}
\end{equation}
This intervention isolates the visual support available for generating $A_{\le t}$ when question information is absent. The equality in Eq.~\eqref{eq:visual_effect} follows because \(V\) is a root node in the assumed SCM; in contrast, QCD and PCD require backdoor adjustment, which distinguishes their interventional estimands from ordinary observational conditioning.
Detailed proofs of the causal effects above are provided in Supplementary Material~\ref{proofs of causal effects}.

\section{Causal Drive Metrics}
Building on the causal effects derived in the preceding section, we introduce a unified causal drive framework for tracing how different information sources shape autoregressive VLM generation. We define three causal drive metrics to measure the source-specific drives of the generated prefix, question input, and visual evidence on the generated answer.
To calibrate the source-specific interventional quantities along the decoding trajectory, we introduce a shared null-source reference:
\begin{equation}
P_0(A_{\le t}=a_{\le t})
=
P(A_{\le t}=a_{\le t}
\mid V=\varnothing,Q=\varnothing).
\label{eq:null_reference}
\end{equation}
Inspired by the log-ratio form of pointwise mutual information (PMI)~\cite{church-hanks-1990-word}, we define each causal drive score as the log ratio of a source-specific interventional probability to the shared  reference. This reference provides a baseline for fixed target sequences and
calibrates the scale of the causal-drive scores. Scores are measured in bits. For each answer, we keep its token sequence fixed and use teacher forcing to evaluate token probabilities under each source condition. Positive and negative values indicate higher and lower likelihood, respectively, under the specified intervention compared to the null-source reference. Sensitivity to the null-image reference is examined in Supplementary Material~\ref{sec:reference_sensitivity}.

\subsection{Prefix Causal Drive (PCD)}
PCD characterizes the influence of the observed generation history on the token produced at decoding step \(t\). It combines the interventional probability obtained by fixing the historical prefix with the shared null-source reference.

\begin{definition}
The PCD at decoding step \(t\) is defined as:
\begin{align}
&\mathrm{PCD}_t
=
\log_2
\frac{
P(A_t=a_t\mid do(A_{<t}=a_{<t}))
}{
P_0(A_{\le t}=a_{\le t})
}
\label{eq:pcd_do}\\
&=
\log_2
\frac{
\begin{aligned}
&\sum_{v,q} P(V=v)P(Q=q\mid V=v) \\
&\quad \times
P(A_t=a_t\mid A_{<t}=a_{<t},V=v,Q=q)
\end{aligned}
}{
P_0(A_{\le t}=a_{\le t})
}.
\label{eq:pcd_adjust}
\end{align}
\end{definition}

Under the intervention $do(A_{<t}=a_{<t})$, the event $A_{<t}=a_{<t}$ has probability one. Therefore, the numerator in Eq.~\eqref{eq:pcd_do} is equivalently $P(A_{\le t}=a_{\le t}\mid do(A_{<t}=a_{<t}))$, making it comparable to the sequence-level null reference while preserving the token-level definition.
The numerator captures the probability of the current token given the fixed observed generation history, while marginalizing over the visual and question contexts. Evaluating PCD over successive decoding steps yields a temporal profile of prefix-conditioned generation during autoregressive decoding. By substituting Eq.~\eqref{eq:prefix_effect} into Eq.~\eqref{eq:pcd_do}, we obtain the computable form in Eq.~\eqref{eq:pcd_adjust}.

\subsection{Question Causal Drive (QCD)} 
QCD quantifies the causal effect of the question \(Q\) on the answer \(A_{\le t}\). It reflects the extent to which the generated content is driven by the question.
\begin{definition}
The QCD at token step \(t\) is defined as:
\begin{align}
&\mathrm{QCD}(A_{\le t}=a_{\le t}\mid Q=q) \\
&= \log_2 \frac{P(A_{\le t}=a_{\le t}\mid do(Q=q))}{P_0(A_{\le t}=a_{\le t})}
\label{eq:qcd_do} \\
&= \log_2\left[
\sum_v P(V=v)\,P(A_1=a_1\mid V=v,Q=q) \right. \\
&\quad \left.\cdot \prod_{i=2}^{t} P(A_i=a_i\mid A_{<i}=a_{<i}, V=v, Q=q)
\right] \\
&- \log_2 P_0(A_{\le t}=a_{\le t}),
\label{eq:qcd_chain}
\end{align}
\end{definition}
We substitute Eq.~\eqref{eq:question_effect} into Eq.~\eqref{eq:qcd_do}. We then apply the autoregressive chain rule to the interventional probability~\cite{murphy2012machine,bengio2003neural}. This gives the computable expression in Eq.~\eqref{eq:qcd_chain}.

\subsection{Visual Causal Drive (VCD)}
VCD quantifies visual support for the generated answer with the question set to the null prompt.
\begin{definition}
The VCD at decoding step \(t\) is defined as:
\begin{align}
&\mathrm{VCD}(A_{\le t}=a_{\le t}\mid V=v) \\
&=
\log_2
\frac{
P(A_{\le t}=a_{\le t}
\mid do(V=v),do(Q=\varnothing))
}{
P_0(A_{\le t}=a_{\le t})
}
\label{eq:vcd_do}\\
&=
\log_2
\frac{
\begin{aligned}
&P(A_1=a_1\mid V=v,Q=\varnothing)\\
&\quad\times
\prod_{i=2}^{t}
P(A_i=a_i\mid A_{<i}=a_{<i},
V=v,Q=\varnothing)
\end{aligned}
}{
P_0(A_{\le t}=a_{\le t})
}.
\label{eq:vcd_chain}
\end{align}
\end{definition}

By substituting Eq.~\eqref{eq:visual_effect} into Eq.~\eqref{eq:vcd_do} and applying the autoregressive chain rule, we obtain Eq.~\eqref{eq:vcd_chain}. VCD traces visual support throughout generation with the question set to the null prompt.
Detailed proofs are in Supplementary Material~\ref{proof of causal drives}. Computational complexity, finite-support estimation, and resource overhead are discussed in Supplementary Material~\ref{compute resources}.

In summary, PCD, QCD, and VCD provide three step-indexed causal-drive metrics for tracing the roles of generation history, question information, and visual evidence throughout autoregressive decoding. Their temporal trajectories characterize how source-specific generation patterns evolve over the course of generation, providing process-level diagnostics beyond final-answer evaluation.

\section{Experiments}
\label{experiments}

In this section, we evaluate the three causal drives: QCD, VCD, and PCD. Our experiments use Qwen3-VL-8B-Instruct, Qwen3-VL-32B-Instruct~\cite{bai2025qwen3}, and InternVL3-8B-hf~\cite{zhu2025internvl3} on the MAVIS~\cite{song2026mavis}  and VLMs Are Biased~\cite{vlms_are_biased} visual question answering datasets.

\subsection{QCD Response to Suggested Answers}
\label{sec:qcd}

We examine whether QCD responds to changes in the answer suggested by a question. We append a suggested answer to the question, replace it with another answer, and measure the change in QCD for the original response. As a control, we replace an unrelated word while keeping the suggested answer unchanged. Both replacements change the same number of tokens.

We use Qwen3-VL-8B-Instruct and Qwen3-VL-32B-Instruct. We randomly select 30 examples from each of the seven topics in VLMs Are Biased, giving 210 examples shared by both models. Each example contains an image, a question, a reference answer (\texttt{ground\_truth}), and an alternative answer that the model may produce by following familiar patterns (\texttt{expected\_bias}).

We first append the reference answer as a suggestion to the original question and generate a response. We then construct two modified questions: one replaces the suggested answer with
\texttt{expected\_bias}, while the other replaces only an additional
unrelated word. We keep the initially generated response fixed and compute the QCD of its explanation under all three versions of the question. Supplementary Material~\ref{app:qcd} provides the construction details.

Figure~\ref{fig:qcd} shows the results for both models. The vertical axis reports QCD before modification minus QCD after modification, so positive values indicate a decrease in the QCD of the original response. Orange represents changes to the suggested answer, and blue represents changes to the unrelated word.

\begin{figure}[t]
    \centering
    \includegraphics[width=0.85\columnwidth]{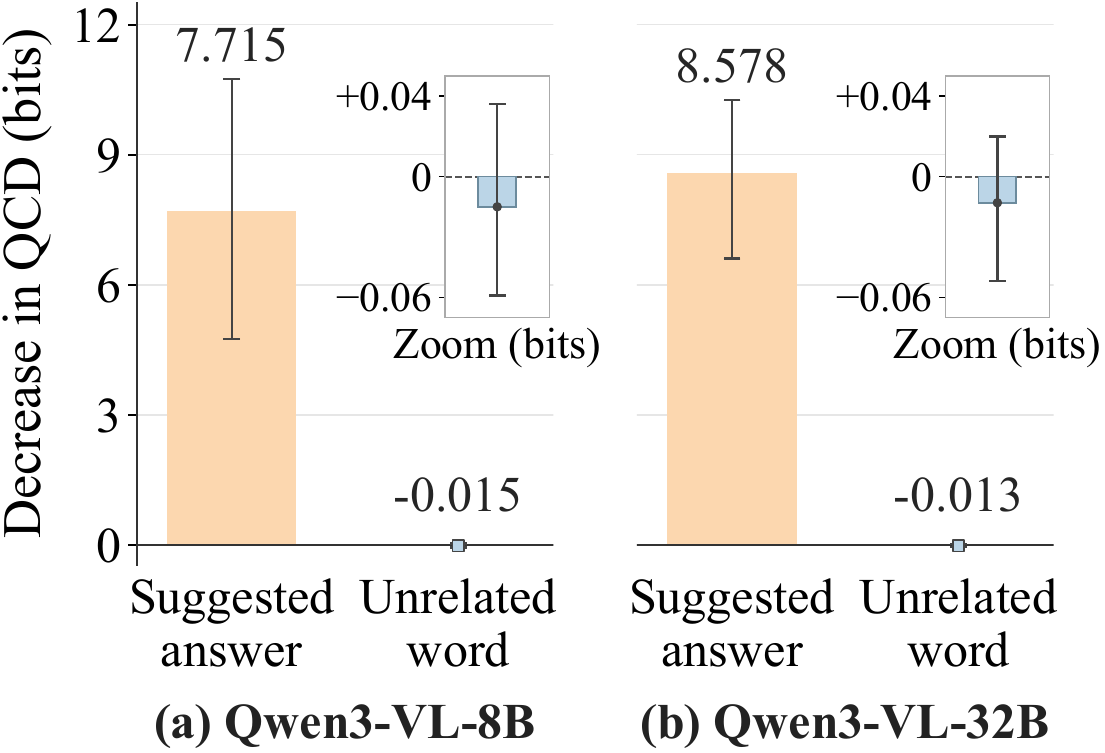}
    \caption{QCD changes in the original explanation after replacing the suggested answer or an unrelated word. Each condition includes all 210 examples. Error bars show 95\% confidence intervals from cluster-resampling the seven topics. Insets use the same expanded scale; dashed lines mark zero.}
    \label{fig:qcd}
\end{figure}

Replacing the suggested answer decreases QCD by 7.715 bits for 8B and 8.578 bits for 32B. Replacing only the unrelated word yields differences of $-0.015$ and $-0.013$ bits, respectively, indicating small increases in QCD; both confidence intervals include zero. For both models, replacing the suggested answer produces a larger QCD decrease than replacing the unrelated word in all seven topics. Thus, with the number of modified tokens held constant, QCD distinguishes changes in the suggested answer from changes in unrelated text. Table~\ref{tab:qcd} reports the paired differences between the two replacements and their confidence intervals.

We also generate new responses to the modified questions and check whether the models adopt the newly suggested answer. We compare the two modifications across all 210 examples. More responses match the newly suggested answer after replacing the suggestion than after replacing only the unrelated word. The difference is 23.33 percentage points for 8B and 28.10 percentage points for 32B.

Next, we examine whether the QCD change of the original response can distinguish examples in which the model subsequently switches to the newly suggested answer. This analysis uses examples whose initial responses match the reference answer: 67 examples for 8B and 73 for 32B. We measure discrimination using AUROC~\cite{FAWCETT2006861}, where higher values indicate better separation between examples that switch and those that do not. QCD achieves AUROCs of 0.889 and 0.838, respectively, compared with 0.657 and 0.598 for the probability baseline (Table~\ref{tab:qcd}). These results show that, within these subsets, QCD changes are also associated with whether the model adopts the newly suggested answer. Supplementary Material~\ref{app:qcd} describes sample selection and answer matching.

\begin{table*}[t]
    \centering
    \caption{QCD differences after modifying the question and prediction of switches to the newly suggested answer. QCD differences use all 210 examples per model. AUROC uses the subsets whose initial responses match the reference answer. The paired difference is the QCD decrease after replacing the suggested answer minus that after replacing the unrelated word.}
    \label{tab:qcd}
    \small
    \setlength{\tabcolsep}{4pt}
    \begin{tabular*}{\textwidth}
        {@{\extracolsep{\fill}}lccccc@{}}
        \toprule
        \multirow{2}{*}{Model}
        & \multicolumn{3}{c}{All 210 examples}
        & \multicolumn{2}{c}{Initially reference-matching subset} \\
        \cmidrule(lr){2-4}
        \cmidrule(lr){5-6}
        & \shortstack{QCD decrease\\Suggested answer}
        & \shortstack{QCD decrease\\Unrelated word}
        & \shortstack{Paired difference\\{}[95\% CI]}
        & \shortstack{Examples /\\Switches}
        & \shortstack{AUROC\\QCD / Probability} \\
        \midrule
        8B
        & $7.715$
        & $-0.015$
        & $7.730\;[4.817,\,10.760]$
        & $67\,/\,44$
        & $0.889\,/\,0.657$ \\
        32B
        & $8.578$
        & $-0.013$
        & $8.592\;[6.604,\,10.275]$
        & $73\,/\,52$
        & $0.838\,/\,0.598$ \\
        \bottomrule
    \end{tabular*}

    \vspace{3pt}
\end{table*}

\subsection{Image Modifications and VCD}
\label{sec:image_modifications}

We examine whether Visual Causal Drive (VCD) captures the effect of image modifications on the original answer. Keeping the generated answer fixed, we blur or replace the image and test whether VCD decreases when visual details become less discernible or the image content changes.

The main experiment uses answers generated by Qwen3-VL-8B-Instruct on 500 image--question pairs from MAVIS. For each sample, we retain the answer generated from the original image and question. We then apply four image modifications: mild blur, strong blur, and two replacements with another image from the same domain. The blur radii are 8 and 27, respectively, allowing us to compare different levels of blur. Image replacement changes the specific visual content while preserving the domain category. We use random seeds 42 and 137 for the two replacements to examine whether the results vary with the selected replacement images. For subsequent VCD scoring under both the original and modified images, we set the question to the null prompt $Q=\varnothing$, following Eq.~\eqref{eq:vcd_chain}. We rescore the same answer tokens using teacher forcing; no new answer is generated after image modification.

Table~\ref{tab:vcd_image_modifications} summarizes the results. The decrease in VCD is defined as the score under the original image minus that under the modified image; a positive value indicates a decrease. We compute 95\% confidence intervals through paired bootstrap resampling over samples~\cite{10.1214/aos/1176344552}. Mild and strong blur reduce mean VCD by 12.79 and 23.22 bits, respectively. Strong blur causes an additional decrease of 10.43 bits compared to mild blur, with a 95\% confidence interval of $[9.71, 11.17]$. Thus, visual causal drive for the original answer decreases more when image details become harder to discern.

The two within-domain image replacements reduce mean VCD by 25.84 and 25.96 bits, respectively, yielding similar results. These decreases show that changing specific image content reduces VCD for the original answer, even when the replacement image belongs to the same domain. More than 80\% of samples exhibit a decrease under each modification, indicating that the effect occurs in the majority of samples.

\begin{table}[t]
    \centering
    \small
    \setlength{\tabcolsep}{3pt}
    \caption{Changes in VCD for the original answers generated by Qwen3-VL-8B-Instruct on 500 MAVIS samples. Mean decrease is measured in bits. CI denotes the 95\% confidence interval. Decrease (\%) is the percentage of samples with lower VCD after image modification. Both replacements use images from the same domain.}
    \label{tab:vcd_image_modifications}
    \begin{tabular}{@{}lrrr@{}}
        \toprule
        Image modification &
        \shortstack{Mean\\decrease} &
        95\% CI &
        \shortstack{Decrease\\(\%)} \\
        \midrule
        Mild blur
        & 12.79 & [11.38, 14.27] & 80.2 \\
        Strong blur
        & 23.22 & [21.59, 24.86] & 93.8 \\
        Replacement (seed $=42$)
        & 25.84 & [24.06, 27.62] & 94.6 \\
        Replacement (seed $=137$)
        & 25.96 & [24.26, 27.72] & 94.2 \\
        \bottomrule
    \end{tabular}
\end{table}

To examine whether the effect of strong blur is confined to a particular type of image, we analyze the four MAVIS domains separately. Mean VCD decreases by 29.07, 20.03, 21.15, and 26.80 bits for Science, Art \& Design, Tech \& Engineering, and Others, respectively. The decrease occurs in all four domains. Sample sizes and confidence intervals are provided in Table~\ref{tab:vcd_domains} of Supplementary Material~\ref{app:image_modifications}.

These results show that VCD captures the effects of reduced visual detail and changes in image content on the original answer. InternVL3-8B-hf exhibits similar trends on the first 200 MAVIS samples; the experimental settings and results are provided in Supplementary Material~\ref{app:image_modifications}, Table~\ref{tab:vcd_internvl}.

\subsection{PCD Response to Coherence with the Prefix}
\label{sec:pcd_coherence}

We examine whether PCD offers greater support for continuations following the preceding topic when the prefix is fixed. We evaluate Qwen3-VL-8B-Instruct and Qwen3-VL-32B-Instruct on the same 100 fixed text sets. Each set corresponds to a MAVIS image--question pair and contains a prefix, one coherent continuation, and two topically unrelated continuations. The prefixes and coherent continuations are manually written in advance. The two unrelated continuations are taken from the coherent continuations of two different sets whose topic families differ from that of the current prefix. We construct two one-to-one assignments. Each continuation appears once after its own prefix as coherent text. It also appears once after each of two other prefixes as unrelated text. Using every continuation in all three conditions balances differences in text difficulty. All texts and assignments are fixed before scoring with the evaluated models, and all 100 sets are included in the analysis. Within each set, the image, question, and prefix remain unchanged; only the continuation varies. Each continuation is scored token by token using teacher forcing.

To compare support for the three continuations, we use a shared reference at each token position: the cumulative probability of the coherent trajectory up to that position, evaluated with an empty image and an empty question. We define the resulting fixed-reference PCD as
\begin{equation}
    \mathrm{PCD}^{\mathrm{ref}}_t(k) = \log_2 \frac{N_t^{(k)}}{D_t^{\mathrm{ref}}}.
    \label{eq:pcd_fixed_reference}
\end{equation}
Here, $t$ is the continuation token position and $k$ indexes the three continuation conditions. $N_t^{(k)}$ is the current-token probability obtained through the earlier described causal adjustment. $D_t^{\mathrm{ref}}$ is the shared reference. The three conditions use the same reference at each position. Their score differences reflect relative support for the continuations.

At each position, we subtract the mean score of the two unrelated continuations from the coherent continuation's score, denoting the difference by $\Delta\mathrm{PCD}^{\mathrm{ref}}$; positive values indicate greater support for coherent content. We first average over the first 64 continuation tokens within each set and then equally across all 100 sets. We compute 95\% percentile bootstrap intervals from 10,000 resamples of complete paired text sets drawn with replacement, taking the 2.5th and 97.5th percentiles of the bootstrap means; curve intervals are pointwise, and the results are conditional on the fixed text bank and pairing scheme.

The mean differences are 0.708 and 0.608 bits for the 8B and 32B models, respectively, with corresponding 95\% confidence intervals of $[0.567, 0.852]$ and $[0.465, 0.752]$, both above zero. The mean difference is positive in 83\% and 79\% of the text sets, respectively.

\begin{figure}[t]
    \centering
    \includegraphics[width=0.9\columnwidth]{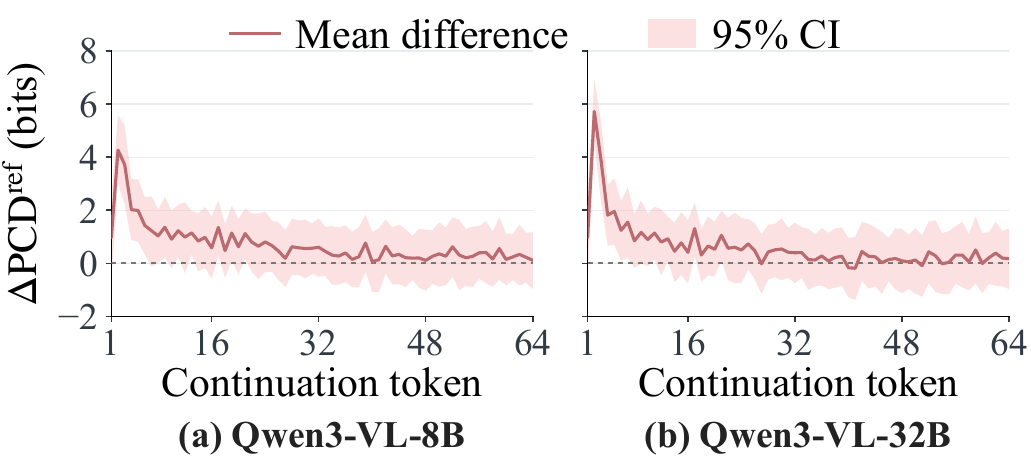}
    \caption{Fixed-reference PCD differences between coherent and unrelated continuations. Curves show the mean difference at each position across all 100 text sets; positive values indicate greater support for coherent continuations. Token 1 is the first position after the shared prefix. Shading shows pointwise 95\% percentile bootstrap intervals.}
    \label{fig:pcd_coherence_difference}
\end{figure}

Figure~\ref{fig:pcd_coherence_difference} shows that the gap is largest near the beginning of the continuation and generally decreases thereafter. The mean differences over the first eight tokens are 2.078 and 2.248 bits, compared with 0.298 and 0.170 bits over tokens 33--64. Within this fixed text bank and pairing scheme, fixed-reference PCD provides greater support for continuations consistent with the preceding topic, with the strongest distinction at the start of the continuation.

\subsection{Causal Drive Trajectories During Generation}
\label{sec:causal_drive_trajectories}

The causal drives of an answer may change as generation proceeds. We track Question Causal Drive (QCD), Visual Causal Drive (VCD), and Prefix Causal Drive (PCD) across answer positions to examine how the measured influence of the question, visual input, and generated prefix evolves.

We evaluate Qwen3-VL-8B-Instruct and Qwen3-VL-32B-Instruct on the same 500 image--question pairs from MAVIS, obtaining 500 answers from each model. To account for differences in answer length, we divide each answer into ten bins by relative generation position. We first compute the mean drive score within each bin of an answer and then average these scores equally across answers.

\begin{figure}[t]
    \centering
    \includegraphics[width=\columnwidth]{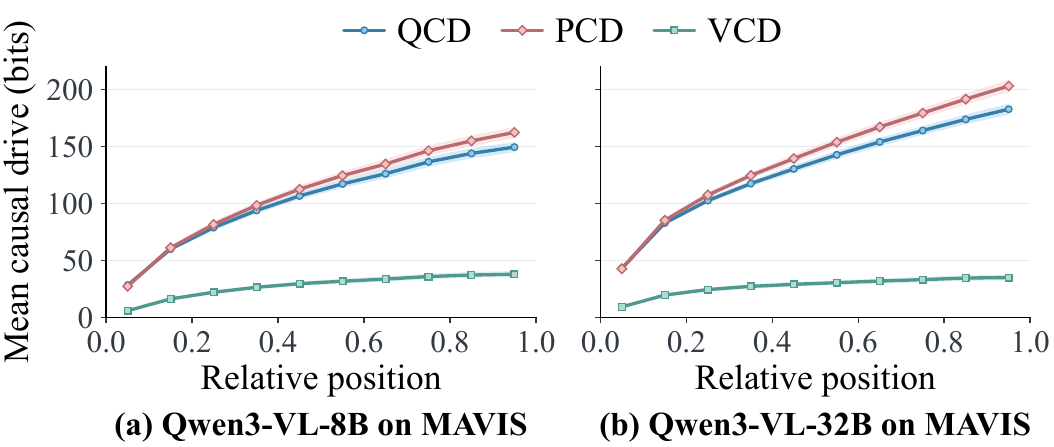}
    \caption{Causal drive trajectories on MAVIS for (a) Qwen3-VL-8B-Instruct and (b) Qwen3-VL-32B-Instruct (500 answers per model). Curves show answer-equal means over ten relative-position bins; shading denotes 95\% bootstrap confidence intervals.}
    \label{fig:causal_drive_trajectories}
\end{figure}

Figure~\ref{fig:causal_drive_trajectories} shows that QCD and PCD increase throughout generation for both models, with PCD exhibiting the largest increase. VCD rises more slowly and tends to level off toward the end of the answer.

\subsection{Sensitivity to Support Size}
\begin{table}[t]
\centering
\small
\setlength{\tabcolsep}{4pt}
\caption{Score differences and response-rank correlations between $K=20$ and $K=40$ across three random seeds. MAD (mean absolute difference) is the average absolute score difference between the two sample sizes for the same responses.}
\label{tab:reference_sample_size}
\begin{tabular}{@{}lcc@{}}
\toprule
Model / metric
& MAD (bits)
& Rank correlation range \\
\midrule
8B / QCD  & 0.648 & 0.99969--0.99977 \\
8B / PCD  & 0.267 & 0.99987--0.99998 \\
32B / QCD & 0.734 & 0.99978--0.99982 \\
32B / PCD & 0.198 & 0.99992--0.99996 \\
\bottomrule
\end{tabular}
\end{table}

The computable expressions for PCD and QCD in Eqs.~\eqref{eq:pcd_adjust} and~\eqref{eq:qcd_chain} require summing over observed contexts (v or q). We denote the number of sampled contexts by $K$. PCD uses $K$ image--question pairs, while QCD uses $K$ images with the target question fixed. We examine whether changing $K$ significantly alters the scores or ranking of responses. We keep the responses generated by the two Qwen3-VL models for the same 100 MAVIS examples fixed. For each of three random seeds, we compute the scores using $K=20$ and $K=40$, then compare the two scores for each response.
For each model and metric, we first average scores over evaluated token positions within each response. We compute the absolute difference between the resulting $K=20$ and $K=40$ scores for each response and seed. We then average these differences equally over 100 responses and three seeds: 42, 137, and 2026. Spearman correlations are computed across responses for each seed; the table reports their minimum–maximum range.
As shown in Table~\ref{tab:reference_sample_size}, increasing $K$ from 20 to 40 yields mean absolute score differences of 0.198–0.734 bits. The rank correlation across the 100 responses exceeds 0.9996 in all cases. Thus, doubling the support size slightly changes the scores and leaves the response rankings nearly unchanged.

\subsection{Qualitative Cases}
\label{sec:qualitative_cases}

Answer--reference overlap and the support an input provides for an answer are distinct dimensions. We compare conventional text-matching scores, causal-drive rankings, and image-perturbation results through three VQA cases from MAVIS. Rankings for each metric are computed within the same 500-answer MAVIS cohort generated by Qwen3-VL-8B-Instruct, where “top” denotes higher values and “bottom” lower values. Blur sensitivity is ranked by the decrease in VCD under strong image blur.

\par\medskip
\noindent\textbf{Case 1: Using image text without explanation.}\par\nopagebreak[4]

\noindent\textbf{Question:} Why is my account not able to place a reservation until Sunday?

\begin{center}
    \includegraphics[width=0.5\columnwidth]{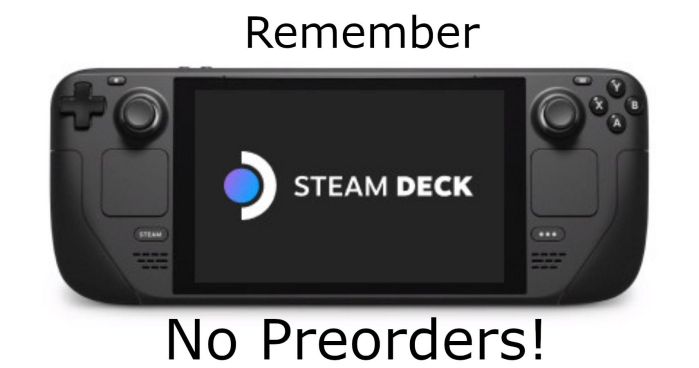}
\end{center}

\noindent\textbf{Answer:} ``The Steam Deck does not accept preorders.''

\noindent\textbf{Reference answer:} ``There is an awareness of potential unauthorized resellers, prompting additional safeguards. A requirement that the reserver has made a purchase on Steam prior to June 2021 has been added. The purchase requirement affects the reservation availability for the first 48 hours.''

The answer repeats ``No Preorders'' from the image but does not explain the safeguards against reselling or the purchase-history requirement in the reference facts. Token F1~\cite{rajpurkar-etal-2016-squad} and ROUGE-L~\cite{lin2004rouge} rank in the bottom 6.6\% and 9.2\%, respectively, while VCD ranks in the top 10.6\%. With the answer fixed, strong image blur reduces VCD. Thus, poor text-matching scores can coexist with substantial visual support for the generated answer.

\par\medskip
\noindent\textbf{Case 2: High text-matching but weak visual drive.}\par\nopagebreak[4]

\begin{wrapfigure}{L}{0.4\columnwidth}
    \centering
    \includegraphics[width=\linewidth]{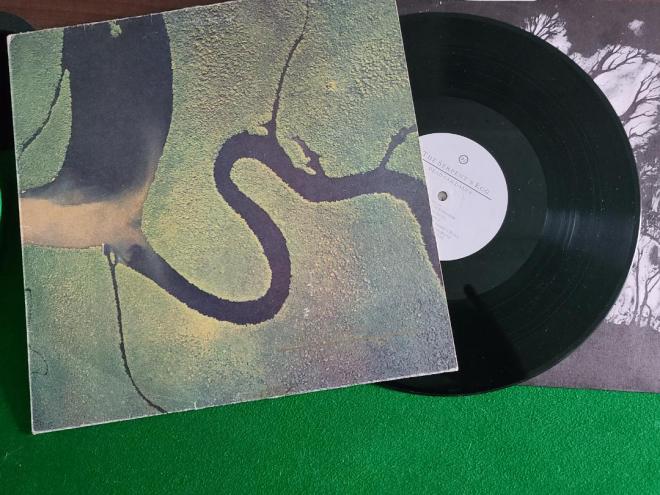}
\end{wrapfigure}

\noindent\textbf{Question:} Was the last Dead Can Dance tour cancelled due to health reasons?

\noindent\textbf{Answer:} ``No, the last Dead Can Dance tour was not cancelled due to health reasons.''

\noindent\textbf{Reference answer:} ``Yes, the last Dead Can Dance tour was cancelled. The cancellation was due to health reasons. The health reasons were affecting one of the band members.''

The answer includes phrases like ``last Dead Can Dance tour,'' ``cancelled,'' and ``health reasons'' from the reference facts. Token F1 and ROUGE-L both rank in the top 0.2\%. However, VCD ranks in the bottom 12.4\%, and its decrease under strong blur ranks in the bottom 9.8\%.
The image shows a record and its sleeve without indicating why the tour was cancelled. The low VCD and limited blur sensitivity indicate weak visual support despite the high textual overlap. Moreover, ``No'' and ``not cancelled'' contradict the affirmative reference facts, showing that high text-matching scores do not establish semantic correctness.

\par\medskip
\noindent\textbf{Case 3: High text matching but weak question drive.}\par\nopagebreak[4]
\noindent\textbf{Question:} Can anyone tell what formula that is?
\begin{wrapfigure}{L}{0.4\columnwidth}
    \centering
    \includegraphics[
width=\linewidth, trim=0 5mm 0 5mm, clip ]{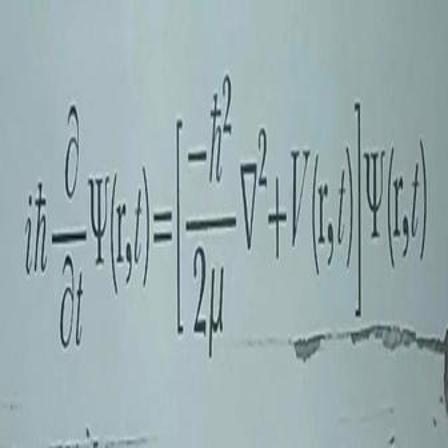}
    \par
\end{wrapfigure}

\noindent\textbf{Answer:} ``That is the time-dependent Schr{\"o}dinger equation in quantum mechanics.''

\noindent\textbf{Reference answer:} ``The formula mentioned is the Schr{\"o}dinger equation for time-dependent particles. The answer hints that the r variable would fit nicely if the equation were used in polar coordinates.''

The response correctly identifies the time-dependent Schr{\"o}dinger equation; Token F1 and ROUGE-L are in the top 5\%. QCD is in the bottom 2.6\%. The question only asks ``what formula that is,'' and the identifying equation is in the image. VCD is in the top 20.8\%, and its drop under strong blur is in the top 11.8\%.

These results indicate visual support for the answer alongside weak question drive. Conventional metrics reflect good agreement with the reference facts, while QCD, VCD, and image-perturbation results provide additional information about the support associated with each input source.

\section{Conclusion}
\label{sec:conclusion}
We introduce a temporal causal drive framework for evaluating VLM generation. Based on an SCM, we derive PCD, VCD, and QCD to quantify generation support from prefixes, visual inputs, and questions. Using only output probabilities, the framework compares source effects along fixed answer trajectories without accessing attention weights or hidden states. Experiments show that QCD distinguishes suggested answer changes from unrelated text changes, VCD decreases under image modifications, and fixed-reference PCD favors coherent continuations. Across the evaluated trajectories, QCD and PCD rise faster than VCD and maintain higher mean scores. Our framework extends VLM evaluation from answer quality to the causal factors underlying generation, revealing source-specific differences and their temporal evolution for model diagnosis.
\subsection*{LLM Usage Statement}

In this work, we used generative AI tools only for language polishing, including improving grammar, clarity, and fluency of the manuscript. We have not used generative AI tools for generating research ideas, developing the methodology, designing experiments, analyzing results, or drawing scientific conclusions; other disclosure categories are not applicable to this work. We reviewed all AI-assisted edits and verified that they did not alter the technical content or scientific claims. We take full responsibility for the final content of this work, including all text, claims, and artifacts.
{
    \small
    \bibliographystyle{ieeenat_fullname}
    \bibliography{main}
}
\clearpage
\maketitlesupplementary

\section{Proofs of Causal Effects}
\label{proofs of causal effects}

The proof of Eq.~\eqref{eq:prefix_effect}, the causal effect of $A_{<t}$ on $A_t$, is as follows:
\begin{proof}
To implement the do-operator \(do(A_{<t}=a_{<t})\), we cut off the original generative mechanism that determines \(A_{<t}\) and replace it with the fixed assignment \(A_{<t}=a_{<t}\). After the intervention, the prefix no longer depends on its upstream variables. Conditioned on this fixed prefix, generation of the current token \(A_t\) depends on \(Q\), \(V\), and \(a_{<t}\). The intervention therefore removes the spurious dependence induced by the backdoor paths \(A_{<t} \leftarrow Q \rightarrow A_t\) and \(A_{<t} \leftarrow V \rightarrow A_t\). Marginalizing over the empirical evaluation distribution $P(V=v)P(Q=q\mid V=v)$ yields Eq.~\eqref{eq:prefix_effect}.
\end{proof}

The proof of Eq.~\eqref{eq:question_effect}, the causal effect of $Q$ on $A_{\le t}$, is as follows:
\begin{proof}
This quantity represents the interventional effect of the question condition \(Q\) on the current answer \(A_{\le t}\). Under the intervention \(do(Q=q)\), the incoming edge \(V\rightarrow Q\) is removed, and \(Q\) is directly fixed to \(q\), no longer depending on its original upstream mechanism. Since \(V\) is a common cause of \(Q\) and \(A_{\le t}\), it confounds the observational relation between them. Therefore, we adjust over the empirical evaluation distribution of \(V\) to block the backdoor path \(Q \leftarrow V \rightarrow A_{\le t}\), yielding Eq.~\eqref{eq:question_effect}.
\end{proof}

The proof of Eq.~\eqref{eq:visual_effect}, the visual-source effect on $A_{\le t}$ with the question fixed to the null prompt, is as follows:
\begin{proof}
The joint intervention fixes the visual variable to \(v\) and the question to the null prompt \(\varnothing\). Intervening on $Q$ removes its incoming edge from $V$, while $V$ has no parent in the proposed SCM. Under this graph and its assumed absence of omitted common causes, the truncated factorization therefore reduces the interventional distribution to the model's conditional generation distribution:
\[
\begin{aligned}
&P(A_{\le t}=a_{\le t}\mid do(V=v),do(Q=\varnothing))\\
&\qquad=P(A_{\le t}=a_{\le t}\mid V=v,Q=\varnothing).
\end{aligned}
\]
We implement $Q=\varnothing$ using the null-question model input. We evaluate the conditional probability on the right-hand side directly through teacher-forced scoring. This probability measures visual support for the observed answer trajectory under the null-question condition.
\end{proof}

\section{Proofs of Causal Drives}
\label{proof of causal drives}

All three drives use the shared null-source reference in Eq.~\eqref{eq:null_reference}. Their computable forms follow from the causal effects above and the autoregressive chain rule.

\paragraph{Prefix Causal Drive (PCD).}
Substituting Eq.~\eqref{eq:prefix_effect} into the PCD definition yields its computable form in the main text. Under $do(A_{<t}=a_{<t})$, the fixed prefix has probability one, so
\[
\begin{aligned}
&P(A_t=a_t\mid do(A_{<t}=a_{<t}))\\
&\quad=P(A_{\le t}=a_{\le t}\mid do(A_{<t}=a_{<t})).
\end{aligned}
\]

\paragraph{Question Causal Drive (QCD).}
For each visual context, the autoregressive chain rule gives
\[
\begin{aligned}
&P(A_{\le t}=a_{\le t}\mid V=v,Q=q)\\
&\quad=\prod_{i=1}^{t}P(A_i=a_i\mid A_{<i}=a_{<i},V=v,Q=q),
\end{aligned}
\]
where $a_{<1}$ is empty. Substituting this factorization into Eq.~\eqref{eq:question_effect}, dividing by the shared null-source likelihood, and taking $\log_2$ yields the QCD expression in the main text. The marginalization over $v$ remains outside the product because one visual context is shared across the full answer sequence.

\paragraph{Visual Causal Drive (VCD).}
We apply the autoregressive chain rule to the conditional probability identified by Eq.~\eqref{eq:visual_effect}, with $Q=\varnothing$. We then divide the result by the shared null-source likelihood and take the base-2 logarithm. This yields the VCD expression in the main text.

\section{Computational Complexity and Resource Overhead}
\label{compute resources}

For a sequence of length $T$, the number of per-token scoring evaluations under exhaustive support enumeration scales as $O(|V|\cdot|Q|\cdot T)$ for PCD, $O(|V|\cdot T)$ for QCD, and $O(T)$ for VCD. With the sampled supports used in the main experiments, PCD scores $K$ image--question pairs and QCD scores $K$ images with the target question fixed; both therefore require $O(KT)$ per-token scoring evaluations. These complexity estimates count token-level scoring evaluations. The cost of each forward pass depends on the model.

Memory usage depends on the sequence length and the number of conditions evaluated together. Storing full logits additionally scales with vocabulary size; batching conditions and retaining target-token probabilities can reduce storage requirements.

\subsection{Finite-support Estimation}
\label{sec:finite_support}

Before computing causal drives, we preprocess the evaluation split to estimate the empirical weights $P(V)$ and $P(Q\mid V)$. For QCD, we fix the target question $q$. We use teacher forcing to score the same generated sequence under each sampled visual context $v$. We then aggregate the resulting probability trajectories with weights $P(V=v)$. For PCD, we score sampled $(v,q)$ pairs and aggregate with $P(V=v)P(Q=q\mid V=v)$. The preprocessing cost is incurred once; repeated teacher-forced scoring under different contexts dominates the evaluation workload. The main text reports the sensitivity of the scores and response rankings to the support size $K$.

\section{Sensitivity to the Null-Image Reference}
\label{sec:reference_sensitivity}

We examine whether the temporal causal-drive profiles depend on the color used to represent the null image. We use the 500 MAVIS answers generated by Qwen3-VL-8B-Instruct for the image-modification experiment, with a generation limit of 64 tokens and support size $K=20$. The default reference is a $448\times448$ black RGB image with pixel values $(0,0,0)$. We replace it with uniform gray $(127,127,127)$ and white $(255,255,255)$ images of the same size, keeping the answer tokens, support, null-question prompt, and all numerator probabilities fixed. Scores under the alternative references are recomputed from the saved token-level probabilities.

For each answer and metric, we compute the Pearson correlation between the default and alternative trajectories at the original evaluation positions $t=1,3,\ldots$, up to the answer length. We report the unweighted arithmetic mean of the 500 answer-level correlations. The 95\% confidence intervals are the 2.5th and 97.5th percentiles of means obtained from 10,000 bootstrap resamples of answers, using seed 42.

\begin{table}[t]
    \centering
    \small
    \setlength{\tabcolsep}{4pt}
    \caption{Mean within-answer Pearson correlations between causal-drive trajectories using the default black reference and an alternative null-image reference. Brackets give 95\% answer-bootstrap confidence intervals across 500 MAVIS answers.}
    \label{tab:null_image_reference}
    \begin{tabular}{@{}llc@{}}
        \toprule
        Reference & Metric & Mean correlation [95\% CI] \\
        \midrule
        Gray  & QCD & 0.99965 [0.99960, 0.99969] \\
        Gray  & VCD & 0.98363 [0.97831, 0.98806] \\
        Gray  & PCD & 0.99972 [0.99968, 0.99975] \\
        \midrule
        White & QCD & 0.99973 [0.99968, 0.99977] \\
        White & VCD & 0.99089 [0.98829, 0.99308] \\
        White & PCD & 0.99979 [0.99975, 0.99981] \\
        \bottomrule
    \end{tabular}
\end{table}

Table~\ref{tab:null_image_reference} shows that QCD and PCD retain mean trajectory correlations above 0.9996 under both alternative references. VCD exhibits lower correlations, particularly with the gray reference, while maintaining a mean correlation above 0.98. The temporal profiles are therefore largely preserved across the three null-image colors.

\section{Experimental Setup and Sample Selection}
\label{app:qcd}

Each example has three versions of the question. The first appends a suggested answer and an unrelated word to the original question, with the suggestion taken from \texttt{ground\_truth}. The second replaces only the suggested answer with \texttt{expected\_bias}, for example, changing \emph{No} to \emph{Yes}. The third retains the original suggestion and replaces only the unrelated word. The image and the original question remain the same in all three versions.

We generate a response to each version to observe whether the answer changes. For QCD computation, we always use the response generated from the first version and compare the QCD of the same explanation under the three questions. The QCD differences in Figure~\ref{fig:qcd} and Table~\ref{tab:qcd} use all 210 examples. The probability baseline measures the change in conditional log-probability of the same explanation before and after modifying the question.

The subsets of 67 and 73 examples in Table~\ref{tab:qcd} are selected according to the initially generated answers. We match each initial response against \texttt{ground\_truth} and \texttt{expected\_bias}. For 8B, 67 responses match \texttt{ground\_truth}, 132 match \texttt{expected\_bias}, and 11 are classified as ``other.'' The corresponding counts for 32B are 73, 124, and 13. The ``other'' category contains responses that cannot be assigned to either answer.

The AUROC analysis includes only examples whose initial responses match \texttt{ground\_truth}. It evaluates whether the QCD change distinguishes examples that switch to \texttt{expected\_bias} after it becomes the suggested answer. Of the 67 eligible 8B examples, 44 switch to that answer; for 32B, 52 of the 73 eligible examples switch. Eligibility depends only on the initial response. Table~\ref{tab:qcd} reports these subset results alongside the QCD differences computed from all 210 examples.

\section{Additional Experiments on Image Modifications}
\label{app:image_modifications}

To examine whether the results under strong blur are consistent across domains, we group the 500 MAVIS samples evaluated with Qwen3-VL-8B-Instruct by domain. Table~\ref{tab:vcd_domains} reports the sample size, mean decrease in VCD, and 95\% confidence interval for each domain. For each sample, the decrease is computed by comparing VCD for the same answer under the original and strongly blurred images.

\begin{table}[t]
    \centering
    \small
    \setlength{\tabcolsep}{3pt}
    \caption{Results under strong blur for Qwen3-VL-8B-Instruct across MAVIS domains. Mean decrease is measured in bits. Confidence intervals are computed through paired bootstrap resampling over samples.}
    \label{tab:vcd_domains}
    \begin{tabular}{@{}lrrr@{}}
        \toprule
        Domain & $N$ & \shortstack{Mean\\decrease} & 95\% CI \\
        \midrule
        Science             &  46 & 29.07 & [22.81, 35.98] \\
        Art \& Design       & 142 & 20.03 & [17.41, 22.79] \\
        Tech \& Engineering & 165 & 21.15 & [18.68, 23.72] \\
        Others              & 147 & 26.80 & [23.39, 30.32] \\
        \bottomrule
    \end{tabular}
\end{table}

We independently generate answers with InternVL3-8B-hf on the first 200 MAVIS samples. We apply mild and strong blur with radii of 8 and 27, respectively. We also replace each image with another image from the same domain using seeds 42 and 137. We keep the original answer tokens fixed and recompute VCD with $Q=\varnothing$. Scoring and confidence intervals follow the procedures used in the main experiment. Table~\ref{tab:vcd_internvl} shows that all four modifications reduce mean VCD. Strong blur causes a larger decrease than mild blur. Both trends agree with the Qwen3-VL-8B-Instruct results.

\begin{table}[t]
    \centering
    \small
    \setlength{\tabcolsep}{4pt}
    \caption{Changes in VCD for the original answers generated by InternVL3-8B-hf on the first 200 MAVIS samples. Mean decrease is measured in bits. CI denotes the 95\% confidence interval. Both replacements use images from the same domain.}
    \label{tab:vcd_internvl}
    \begin{tabular}{@{}lrr@{}}
        \toprule
        Image modification & \shortstack{Mean\\decrease} & 95\% CI \\
        \midrule
        Mild blur                & 10.65 & [8.23, 13.48] \\
        Strong blur              & 13.85 & [11.22, 16.93] \\
        Replacement (seed $=42$)  & 25.09 & [22.19, 28.48] \\
        Replacement (seed $=137$) & 25.07 & [22.21, 28.23] \\
        \bottomrule
    \end{tabular}
\end{table}

\end{document}